\documentclass[11pt,a4paper]{article}
\usepackage{times,latexsym}
\usepackage{url}
\usepackage[T1]{fontenc}

\usepackage[acceptedWithA]{tacl2018v2}
\usepackage{amsfonts}       
\usepackage{amsmath}
\usepackage{amsthm}
\usepackage{nicefrac}       
\usepackage{xfrac}
\usepackage{subcaption}
\usepackage{multirow}
\usepackage{ctable} 
\usepackage{booktabs}       
\usepackage{diagbox}
\usepackage{csquotes}

\newtheorem{theorem}{Theorem}[section]
\newtheorem{prop}{Proposition}[section]

\newtheorem{defn}{Definition}[section]

\newtheorem{lemma}{Lemma}[section]

\newtheoremstyle{claim}
  {\topsep}
  {\topsep}
  {\itshape}
  {}
  {\itshape}
  {.}
  {.5em}
  {\thmname{#1}\thmnumber{ #2}\thmnote{ (#3)}}
\theoremstyle{claim}
\newtheorem*{claim}{Claim}

\DeclareMathOperator*{\argmin}{arg\,min}

\title{Modeling Non-Cooperative Dialogue: Theoretical and Empirical Insights}

\author{
Anthony Sicilia \\
Intelligent Systems Program, \\
University of Pittsburgh, \\
Pittsburgh, USA \\
{\tt anthonysicilia@pitt.edu} \\

\And

Tristan Maidment \\
Intelligent Systems Program, \\
University of Pittsburgh, \\
Pittsburgh, USA \\
{\tt tristan.maidment@pitt.edu} \\

\AND

Pat Healy \\
Department of Informatics \\
and Networked Systems, \\
University of Pittsburgh, \\
Pittsburgh, USA \\
{\tt pat.healy@pitt.edu} \\

\And

Malihe Alikhani \\
Department of Computer Science and \\
Intelligent Systems Program, \\
University of Pittsburgh, \\
Pittsburgh, USA \\
{\tt malihe@pitt.edu} \\
}

\date{}

\begin{document}
\maketitle

\begin{abstract}
Investigating cooperativity of interlocutors is central in studying pragmatics of dialogue. Models of conversation that only assume cooperative agents fail to explain the dynamics of strategic conversations. 
Thus, we investigate the ability of agents to identify non-cooperative interlocutors while completing a concurrent visual-dialogue task. Within this novel setting, we study the optimality of communication strategies for achieving this multi-task objective. We use the tools of learning theory to develop a theoretical model for identifying non-cooperative interlocutors and apply this theory to analyze different communication strategies. 
We also introduce a corpus of non-cooperative conversations about images in the \textit{GuessWhat?!} dataset proposed by \citet{de2017guesswhat}. We use reinforcement learning to implement multiple communication strategies in this context and find empirical results validate our theory.
\footnote{Per-print version to appear in \textit{Transactions of the Association for Computational Linguistics} published by MIT Press.}
\end{abstract}

\section{Introduction}
\label{sec:intro}
A robust dialogue agent cannot always assume a cooperative conversational counterpart when deployed \textit{in the wild}. Even in goal-oriented settings, where the intent of an interlocutor may seem to be granted, bad actors and disinterested parties are free to interact with our dialogue systems. These non-cooperative interlocutors add harmful noise to data, which can elicit unexpected behaviors from our dialogue systems. Thus, the need to study non-cooperation increases daily as we build and deploy conversational systems which interact with people from different demographics, political views, and intents, continuously learning from the collected data. Examples include Amazon Alexa, task-oriented systems that help patients recovering from injuries or can teach a person a new language, and systems that help predict deceptive behaviors in courtrooms. To effectively communicate in presence of unwanted behaviors like bullying~\cite{cercas-curry-rieser-2018-metoo}, systems need to understand users' strategic behaviors \cite{asher2013strategic} and be able to identify non-cooperative actions.
%
%
Designing agents that learn to identify non-cooperative interlocutors is challenging since it requires processing the context of the dialogue in addition to modeling the choices that interlocutors make under uncertainty -- choices which typically affect their ability to complete tasks unrelated to identifying non-cooperation as well.
In light of this, we ask: 
\begin{displayquote}
\textit{What communication strategies are effective for identifying non-cooperative interlocutors, while also achieving the goals of a distinct dialogue task?}
\end{displayquote}

To answer this question, we appeal to a simple non-cooperative version of the visual dialogue game \textit{Guess What?!} \cite{de2017guesswhat}. See Figure~\ref{fig:intro-example} for an example. The game consists of a multi-round dialogue between two players: a \textit{question}-player and an \textit{answer}-player. Both have access to the same image whereas only the answer-player has access to an image-secret; i.e., a particular goal-object for the question-player to recognize. The question-player's goal is to ask the answer-player questions which will reveal the secret. A \textit{cooperative} answer-player then provides good answers to assist in this goal. In the original game, the answer-player is always cooperative. Our modified game instead allows the answer-player to be \textit{non-cooperative} with some non-zero probability. 
Unlike a cooperative answer-player, a non-cooperative answer-player will not necessarily act in assistance to the question-player, and instead, may attempt to reveal an incorrect secret or otherwise hinder information exchange. In experiments, the specific strategies we study are learned from human non-cooperative conversation. 
The question-player, importantly, does not know if answer-player is non-cooperative. At the end of the dialogue, the question-player's final objective is not only to identify the goal-object, but also to determine if the conversation takes place with a cooperative or non-cooperative answer-player.
\begin{figure}
    \centering
    \includegraphics[width=.99\columnwidth]{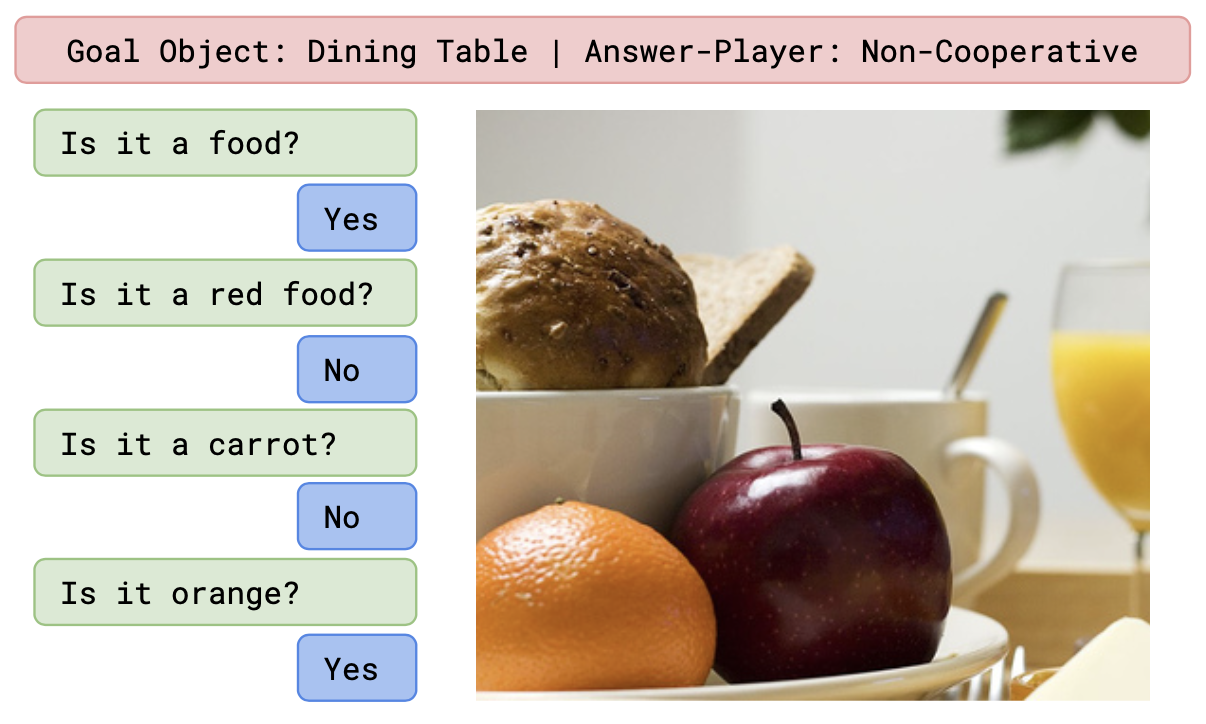}
    \caption {\small{(Example) The question-player's objective is to identify a secret goal-object (the dining table). The answer-player, who may be cooperative or non-cooperative, gives binary responses to the question-player's queries. In this example, the answer-player is non-cooperative and leads the question-player to an incorrect object (the orange). This is a real example produced by autonomous agents (described in Section~\ref{sec:results}).}}
    \label{fig:intro-example}
\end{figure}

We propose a formal theoretical model for analyzing communication strategies in the described scenario. We frame the question-player's objective in terms of two distinct classification tasks and use tools from the theory of learning algorithms to analyze relationships between these tasks. Our main theoretical result identifies circumstances where the question-player's performance in identifying non-cooperation  correlates with performance in identifying the goal-object. Building on this, we provide a mathematical definition of the \textit{efficacy} of a non-cooperative player which is based on the conceptual idea that cooperation is necessary to make progress in dialogue. Our analysis concludes that when the answer-player is \textit{effective} in this sense, the question-player can gather useful information for both the object identification task and the non-cooperation identification task by selecting a communication strategy based \textit{only} on the former objective. 

To test the assumptions of our theoretical model as well as the value of the aforementioned communication strategy in practice, we implement this strategy using reinforcement learning (RL). Our experiments validate our theory. As compared to heuristically justified baselines, the communication strategy motivated by our theory yields consistently better results. To conduct this experiment, we have collected a novel corpus of non-cooperative \textit{Guess What?!} game instances which is \href{https://github.com/anthonysicilia/modeling-non-cooperation-TACL2022}{publicly available}.\footnote{\url{https://github.com/anthonysicilia/modeling-non-cooperation-TACL2022}} Throughout experimentation, we provide a qualitative and quantitative analysis of the non-cooperative strategies present in our corpus. These results, in particular, demonstrate that non-cooperative autonomous agents that utilize dialogue history can better deceive question-players. This contrasts the observation of \citet{strub2017end} that cooperative answer-players do not use this information.

In total, our work is positioned at the intersection of two foci: \textit{detection} of non-cooperative dialogue and \textit{modeling} of non-cooperative dialogue. Unlike many \textit{detection} works, we consider detection in context of interaction. Additionally, while many \textit{modeling} works consider the intent of conversational agents and construct strategies for non-cooperative dialogue based on this, our strategies are motivated purely from a learning theoretic argument. As we are aware, a theoretical description similar to ours has not been given before.


\section{Related Works}
\label{sec:related}
The view that conversation is not necessarily cooperative is not novel, but the argument can be made that it has lacked sufficient investigation in the dialogue literature \cite{lee2000ethics}. Game theoretic investigations of non-cooperation are plentiful, perhaps beginning with work of \citet{nash1951non}. Concepts from this space, such as the stochastic games introduced by \citet{shapley1953stochastic}, have been used to model dialogue \cite{barlier2015human} when non-cooperation between parties is allowed. \citet{pinker2008logic} also consider a game-theoretic model of speech. In fact, even the dialogue game we consider in this text can be modelled through game-theoretic constructs; e.g., a Bayesian Game \citep{kajii1997robustness}. Whereas game theory focuses primarily on analysis of strategies, studying non-cooperation in dialogue requires both the learning of strategies and the learning of utterance meaning. Aptly, our use of the theory of learning algorithms (rather than game theory) is suited to handle both of these. While we are first to use learning theory, efforts to characterize non-cooperation in dialogue, learn non-cooperative strategies in autonomous agents, and detect non-cooperation in dialogue are not absent from the literature \cite{pluss2010non,georgila2011learning,shim2013taxonomy, vourliotakis2014detecting}. 
We discuss these topics in detail in the following.
\paragraph{Modeling Non-Cooperative Dialogue.}
One of the earliest works on non-cooperation -- specific to dialogue -- is that of \citet{jameson1994cooperating} which considers strategic conversation for advantage in commerce. Similarly, \citet{traum2008virtual} focuses on negotiation and \citet{georgila2011reinforcement} focus on learning negotiation strategies (i.e., argumentation) through reinforcement learning (RL). More recently, \citet{efstathiou2014learning} consider using RL to teach agents to compete in a resource-trading game and \citet{keizer2017evaluating} use \textit{deep} RL to model negotiation in a similiar game. In most of these, the intent of interlocutors is assumed and utilized in model design. 
In the last, strategies are learned from data similarly to our work, but objectives for learning are not motivated by learning-theoretic analysis as in ours.
\paragraph{Detecting Non-Cooperative Dialogue.}
The work of \citet{zhou2004automating} presents an early example of automated deception detection which focuses on indicators arising from the used language. \citet{pluss2014computational} also focuses on how (more general) non-cooperative dialogue can be identified at a linguistic level. 
Besides linguistic cues, several works employ additional features in identification of deception. These include physiological responses \cite{abouelenien2014deception}, human micro-expressions \cite{wu2017deception}, and acoustics \cite{levitan2019deception}. There are also many novel scenarios for detection of deception including talk-show games \cite{soldner2019box}, interrogation games \cite{chouyour}, and news \cite{conroy2015automatic, shu2017fake}.
\paragraph{Other Visual Dialogue Games.}
As \citet{galati2020retained} observe, conversation involving multiple media for information transfer (instead of a single medium) typically leads to increased understanding between interlocutors. Thus, visual-dialogue is a particularly interesting setting for investigating both cooperation and non-cooperation. Appropriately, cooperative visual-dialogue games \cite{das2017visual, schlangen2019grounded,haber-etal-2019-photobook} are a growing area of study. We extend, in particular, the cooperative game \textit{Guess What?!} proposed by \citet{de2017guesswhat} to explicitly allow for non-cooperation. Whereas visual-dialogue research often focuses on mechanisms to improve task success, our work is more broadly interested in an analysis of human communication strategies within a non-cooperative, multi-task setting.
\paragraph{Related Learning Theoretic Work.}
Classification of non-cooperative examples is similar to detection of adversarial examples; see \citet{serban2018adversarial} for a survey. Still, most learning-theoretic work only discusses models which are robust to adversaries; e.g., see \citet{feige2015learning, cullina2018pac, attias2019improved, bubeck2019adversarial, diochnos2019lower}; and \citet{montasser2020reducing} to name a few. In contrast, we focus on detection. Additionally, our theoretical results are more broad and do not explicitly model adversarial intent. Identifying non-cooperation in dialogue is also related to detecting distribution shift in high-dimensional, distribution-independent settings \cite{gretton2012kernel, lipton2018detecting, rabanser2018failing, atwell-etal-2022-change} as well as learning to generalize in presence of such distribution shift \cite{ben2010theory, ganin2015unsupervised, zhao2018adversarial, zhao2019learning, schoenauer2019multi, johansson2019support, germain2020pac, sicilia2022pac}. This connection is a strong motivation for our theoretical work, but we emphasize our results are \textit{not} a trivial application of existing theory. 
\section{Dataset}
\label{sec:data}
In this section, we first describe our modified version of the \textit{GuessWhat?!} game. Then, we describe the data acquisition process as well as the non-cooperative dataset used in this study. The dataset will be made publicaly available upon publication.

\subsection{Proposed Dialogue Game}
As noted, our proposed dialogue game is a modification of the cooperative two-player visual-dialogue game \textit{GuessWhat?!} \citep{de2017guesswhat}. Distinctly, our version incorporates non-cooperation.
\paragraph{Initialization.}
An image is randomly selected and an object within this image is randomly chosen to be the goal-object. With some probability, the game instance is designated as a \textit{cooperative} game. Otherwise, the game is \textit{non-cooperative}.
\paragraph{Players.}
Unlike the original \textit{GuessWhat?!} game, there are three (not two) player roles: the question-player, the \textit{cooperative} answer-player, and the \textit{non-cooperative} answer-player. For \textit{cooperative} game instances (decided at initialization), the cooperative answer-player is put in play. Otherwise, the non-cooperative answer-player is put in play. The question-player always plays and does not know whether the answer-player is cooperative or non-cooperative. To start,  all active players are granted access to the image. The question-player asks yes/no questions about the image and objects within the image. At the end of dialogue, the question-player will use the gathered information to guess both the \textit{unknown} goal-object and the (cooperation) type of the active answer-player.\footnote{This is done simultaneously, so knowledge of the correctness of one guess cannot inform the other guess.} Unlike the question-player, the active answer-player has knowledge of the game's goal-object and responds to the question-player's queries with \textit{yes}, \textit{no}, or \textit{n/a} (not applicable).
\paragraph{Objectives.}
The question-player's goals are always to identify both the goal-object and the presence of non-cooperation if it exists (i.e., if the non-cooperative answer-player is in play). The cooperative answer-player's goal is to reveal the goal-object to the question-player by answering the yes/no questions appropriately. The non-cooperative answer-player's goal is instead to lead the question-player away from this goal object; i.e., to ensure the question-player does not correctly guess this object. There is no specific way in which this \textit{misleading} must be done (e.g., there is not always an alternate object). Instead, during data collection, participants are simply instructed to \textit{deceive} the question-player.
\paragraph{Gameplay.}
The question-player and active answer-player converse until the question-player is ready to make a guess or a pre-specified maxium number of dialogue rounds have transpired.\footnote{For data collection, no question limit is set. Experiments in Section~\ref{sec:results} follow \citet{strub2017end} and set the max to 5.} The question-player is then presented with a list of possible objects and must guess which of these was the secret goal-object. In addition, the question-player must guess whether the answer-player was cooperative or non-cooperative.

\subsection{Data Collection}
\paragraph{Collection.} We developed a web application to collect dialogue from human participants taking the role of a non-cooperative answer-player. Participants were native English speakers recruited via an online crowd-sourcing platform and paid \$15 per hour according to our institution's human subject review board.
Participants were asked to deceive
an autonomous question-player pre-trained to identify the goal-object only. For pre-training, we used the original \textit{Guess What?!} game corpus and supervised learning setup \citep{de2017guesswhat, strub2017end}. Participants received an image and a crop that indicated the goal-object. Both of these are randomly sampled from the original \textit{Guess What?!} game corpus. They were tasked with leading the question-player away from this goal object by answering questions with \textit{yes}, \textit{no}, or \textit{n/a}. Dialogue persisted until the question-player made a guess. 
\paragraph{Dataset.}
We collected 3746 non-cooperative dialogues. Dataset statistics are shown in Table~\ref{tab:datastats}, while visualization of the object and dialogue-length distributions are shown in Figure~\ref{fig:datastats}. Compared to the original \textit{Guess What?!} corpus, both dialogue-length and object distributions are similar. For objects, this is expected as these are uniformly sampled from the original corpus. We see 16 of our 20 most likely objects are shared with the 20 most likely of the original \textit{GuessWhat?!} object distribution, and further, the first 4 objects have identical ordering (see Figure~\ref{fig:orig-datastats}). Differences here are simply attributed to randomness and the increasing uniformity as likelihood of an object decreases. For dialogue length, one might expect non-cooperative dialogue to be longer. Instead, the distributions are both right-skew with an average near 5 (i.e., 4.99 in our dataset and 5.11 in the original \textit{GuessWhat?!} corpus). The primary difference is that the original corpus has more outliers, which is most probably a result of the increased sample size. We likely observe consistency between our non-cooperative corpus and the original corpus because the question-player -- who controls dialogue length -- is autonomous and trained on a cooperative corpus. Hence, this and other aspects of our non-cooperative corpus may be influenced by pre-conditioning the question-player for cooperation. This issue is mitigated in our experiments (Section~\ref{sec:results}) where the question-player is also trained on simulated non-cooperative dialogue. Also note, while the size of our collected dataset is smaller than the original cooperative corpus, we only use our data to train an autonomous, non-cooperative answer-player. When a larger sample is required (e.g., when training the question-player via RL), we use simulated non-cooperative data generated by the pre-trained, non-cooperative answer-player, which is a standard technique in the literature \citep{strub2017end}.

Besides the statistics shown in Table~\ref{tab:datastats} and Figure~\ref{fig:datastats}, we also point out the question-player succeeded at identifying the goal-object in only 19\% of the collected games. Comparatively, on an autonomously generated and fully cooperative test set, comparably trained question-players achieve 52.3\% success \citep{strub2017end}. This indicates the deceptive strategies employed by the humans were effective at fooling the question-player to select the wrong goal-object. More detailed analysis of the strategies used by the participants is given in Section~\ref{sec:results}; these strategies are self-described by the participants and also automatically detected for a simple case. Finally, we also computed the answer distribution on the collected corpus: answers were 46\% \textit{yes}, 52\% \textit{no}, and 2\% \textit{n/a}.

\begin{table}
    \centering\footnotesize
    \begin{tabular}{cccccc}
         & \textbf{images} & \textbf{objects} &  \textbf{words (+3)} & \textbf{questions} \\ \toprule
         Ours & 2.7K & 2.8K & 2.3K (1K) & 8.1K \\
         GW & 67K & 134K & 19K (6.6K) & 277K
    \end{tabular}
    \caption{\small Count of unique images, objects, words, and questions within the non-cooperative games collected. (+3) gives count of words with at least 3 occurences. First row is our proposed dataset. Second (GW) reports computed stats on the original \textit{GuessWhat?!} corpus.}
    \label{tab:datastats}
\end{table}

\begin{figure*}
    \centering
    \includegraphics[width=\textwidth]{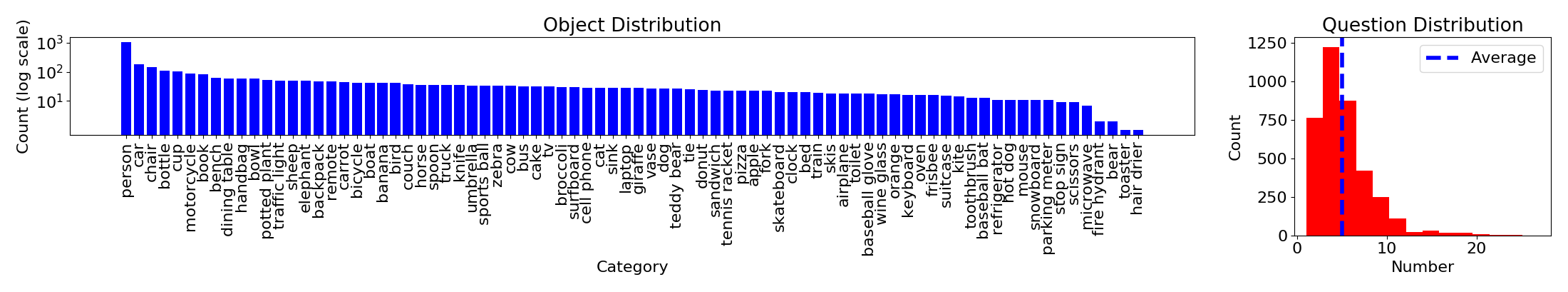}
    \caption{\small Our new non-cooperative dataset. \textbf{Left} shows distribution of objects in the collected games. All 80 objects in the original \textit{GuessWhat?!} corpus occur. \textbf{Right} shows distribution of question-counts per dialogue.}
    \label{fig:datastats}
\end{figure*}
\begin{figure*}
    \centering
    \includegraphics[width=\textwidth]{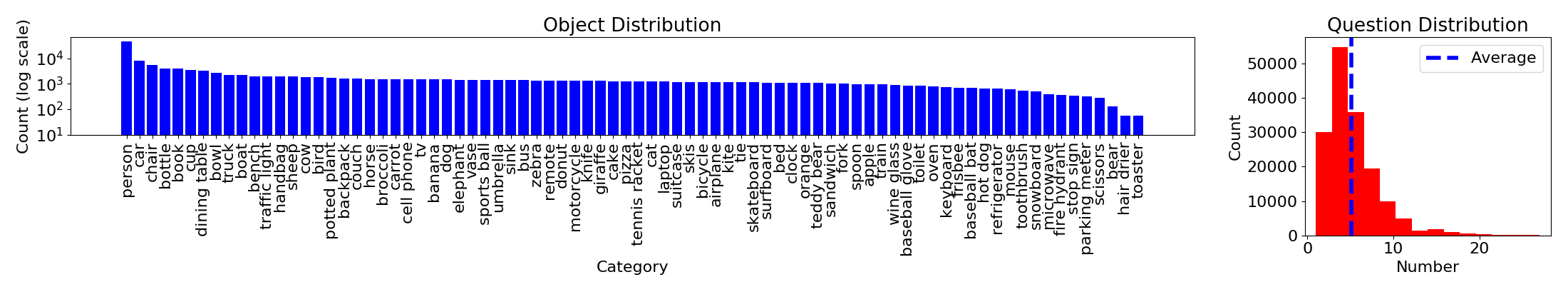}
    \caption{\small Original \textit{GuessWhat?!} dataset. \textbf{Left} shows distribution of objects in original games. \textbf{Right} shows distribution of question-counts per dialogue with 114 outliers larger than 27 removed for improved visualization.}
    \label{fig:orig-datastats}
\end{figure*}
\section{A Theoretical Model}
\label{sec:theory}
This section formally models the objectives of the question-player as two distinct learning tasks. We use results from the theory of learning algorithms to give a relationship between these tasks in Thm~\ref{thm:main}. We then use Thm~\ref{thm:main} to analyze communication strategies in Section~\ref{sec:comm_analysis}.
\subsection{Setup}
\label{sec:theory_setup}
As described in Section~\ref{sec:data}, the question-player has two primary objectives: identification of the goal-object and identification of non-cooperation. To do so, the question-player is granted access to the image and may also converse with an answer-player. In the end, the question-player guesses based on this evidence (i.e., the image features and dialogue history). Mathematically, we encapsulate the question player's guess as a \textit{learned hypothesis} (i.e., function) from the game features to the set of object labels or the set of cooperation labels. 
\paragraph{Key Terms.}
We write $\mathcal{Y}$ to describe the finite set of object labels and $\mathcal{Z} = \{\mathtt{CP}, \mathtt{NC}\}$ for the set of cooperation labels; $\mathtt{CP}$ denotes cooperation and $\mathtt{NC}$ denotes non-cooperation. In relation to the example in Figure~\ref{fig:intro-example}, $\mathcal{Y}$ might contain labels for the orange, apple, cups, and dining-table. In the same example, the cooperation label would be $\mathtt{NC}$ to indicate a non-cooperative answer-player. We use $\mathcal{X}$ to denote the feature space which contains all possible game configurations. For example, each $X \in \mathcal{X}$ might capture the dialogue history, the image, and particular features of the image pre-extracted for the question-player (i.e., which objects are contained in the image at which locations). With this notation, the question-player's learned hypotheses may be described as an \textit{object identification hypothesis} $ o: \mathcal{X} \to \mathcal{Y}$ and a \textit{cooperation identification hypothesis} $c: \mathcal{X} \to \mathcal{Z}$. The question-player learns these functions by example. In particular, we assume the question-player is given access to a random sequence of $m$ examples $S = (X_i, Y_i, Z_i)_{i=1}^m$ independently and identically distributed according to an unknown distribution $\mathbb{P}_\theta$ over $\mathcal{X} \times \mathcal{Y} \times \mathcal{Z}$. To abbreviate, we write $S \overset{\mathrm{iid}}{\sim} \mathbb{P}_\theta$ and assume all samples are of size $m$ for simplicity. The distribution $\mathbb{P}_\theta$ is dependent on the question-player's \textit{communication policy} $\pi_\theta$ which we assume is uniquely determined by the real-vector $\theta$. Later, this allows us to select communication strategies using common reinforcement learning algorithms.

We emphasize the dependence of $\mathbb{P}_\theta$ on $\pi_\theta$ distinguishes our setup from typical scenarios in learning theory. Besides learning the hypotheses $o$ and $c$, the question-player
can also select the communication policy $\pi_\theta$. This policy implicitly dictates the distribution over which the question-player learns, and thus, can either improve or hurt the player's chance at success. As in reality, neither we nor the learner have  knowledge of the mechanism through which changes to the communication policy $\pi_\theta$ modify the distribution $\mathbb{P}_\theta$. Our only assumption is that changing $\pi_\theta$ does not modify the probability of cooperation. That is, there is a constant $\mathbf{p}_\mathtt{NC} \in (0,1)$ such that for all $\pi_\theta$
\begin{equation}\label{eqn:only_pi_assump}\footnotesize
\mathbf{Pr}(Z = \mathtt{NC}) = \mathbf{p}_\mathtt{NC}; \quad (X,Y,Z) \sim \mathbb{P}_\theta.
\end{equation}
This agrees with the description in Section~\ref{sec:data} where the game instance is designated cooperative or non-cooperative prior to dialogue. With a random sample $S$, an unbiased estimate for $\mathbf{p}_\mathtt{NC}$ is
\begin{equation}\footnotesize
\widehat{\mathbf{p}}_S \overset{\mathrm{def}}{=} \tfrac{1}{m}\sum\nolimits_i \mathbf{1}[Z_i = \mathtt{NC}]
\end{equation}
where $\mathbf{1}$ is the indicator function.
\paragraph{Error.}
To measure the quality of the question-player's guesses, we report the observed error-rate on the sample $S = (X_i, Y_i, Z_i)_{i=1}^m$. In particular, the empirical \underline{\textbf{o}}bject-identification \underline{\textbf{er}}ror for any hypothesis $o : \mathcal{X} \to \mathcal{Y}$ is defined
\begin{equation}\footnotesize
    \widehat{\mathbf{oer}}_S(o) \overset{\mathrm{def}}{=} \tfrac{1}{m}\sum\nolimits_{i=1}^m \mathbf{1}[o(X_i) \neq Y_i].
\end{equation}
 Similarly, the \underline{\textbf{c}}ooperation identification \underline{\textbf{er}}ror for any hypothesis $c : \mathcal{X} \to \mathcal{Z}$ is defined
\begin{equation}\footnotesize
    \widehat{\mathbf{cer}}_S(c) \overset{\mathrm{def}}{=} \tfrac{1}{m}\sum\nolimits_{i=1}^m \mathbf{1}[c(X_i) \neq Z_i].
\end{equation}
In some cases, we instead restrict the sample over which we compute the empirical object-identification error. Specifically, we restrict to cooperative game instances and write
\begin{equation}\footnotesize
    \widehat{\mathbf{oer}}_{S}(o \mid \mathtt{CP}) \overset{\mathrm{def}}{=} \widehat{\mathbf{oer}}_{S'}(o)
\end{equation}
where $S' = ((X_i, Y_i) \mid Z_i = \mathtt{CP})$ is the sample $S$ with each triple where $Z_i \neq \mathtt{CP}$ removed. The case $ \widehat{\mathbf{oer}}_{S}(o \mid \mathtt{NC})$ is defined similarly. Based on these, we further define the \textit{cooperation gap}
\begin{equation}\footnotesize
    \Delta_S(o) \overset{\mathrm{def}}{=} \widehat{\mathbf{p}}_S \cdot \widehat{\mathbf{oer}}_{S}(o | \mathtt{NC}) - (1 - \widehat{\mathbf{p}}_S)  \cdot \widehat{\mathbf{oer}}_S(o | \mathtt{CP}).
\end{equation}
This gap describes observed change in (weighted) object-identification error induced by change in cooperation. We often expect $\Delta$ to be positive.\footnote{Delta is negative when the object-identification error is higher on cooperative examples than non-cooperative examples (for simplicity, this assumes $\hat{\mathbf{p}}_S = 0.5$). In practice, we rarely expect cooperation to lead to worse performance.} 

Finally, recall $S \overset{\mathrm{iid}}{\sim} \mathbb{P}_\theta$ and $\mathbb{P}_\theta$ is unknown, so in practice, we can only report the observed error discussed above. Still, we are typically more interested in the \textit{true} or \textit{expected} error for future samples from $\mathbb{P}_\theta$. This quantity tells us how the question-player's hypotheses generalize beyond the random samples we observe. Precisely, the expected cooperation-identification error of a hypothesis $c: \mathcal{X} \to \mathcal{Z}$ is defined
\begin{equation}\footnotesize
    \mathbf{cer}_\theta(c) \overset{\mathrm{def}}{=} \mathbf{E}[\widehat{\mathbf{cer}}_S(c)] = \mathbf{Pr}(c(X) \neq Z)
\end{equation}
where $(X, Y, Z) \sim \mathbb{P}_\theta$. The true (or expected) object-identification error is similarly defined.
\subsubsection{Applicability to Distinct Contexts} 
\label{sec:setup_applicab}
While we have specified our discussion above to promote understanding, one of the benefits of our theoretical framework is that it is fairly general. In fact, the reader may be concerned that our discussion above lacks precise definitions of seemingly important terms; i.e., the feature space $\mathcal{X}$ and the communication policy $\pi_\theta$. These components are intentionally left abstract because our theoretical results make no assumptions on the mechanism through which $\pi_\theta$ influences $\mathbb{P}_\theta$ -- i.e., except Eq.~\eqref{eqn:only_pi_assump}. Further, our results make no assumptions on how the game configurations are represented in the feature space $\mathcal{X}$. This space could correspond to any set of dialogues with/without some associated data (e.g., images). Lastly, the only assumptions on the label spaces are that $\mathcal{Y}$ is finite and $\mathcal{Z}$ is binary. In this sense, our theoretical discussion is applicable to very general scenarios beyond the simple visual-dialogue game considered. We emphasize some examples later in Section~\ref{sec:discussion}.
\subsection{Bounding Cooperation Identification Error}
To motivate our main result, we informally observe that identifying non-cooperation is essentially a problem of identifying distribution-shift. Specifically, we are interested in differences between the two dialogue distributions induced by cooperative and non-cooperative answer-players, respectively.
Luckily, there is a rich literature on the topic of distribution-shift. We take insight, in particular, from work of \citet{ben2007analysis, ben2010theory} which measures shift using the \textit{symmetric difference hypothesis class}. For a set of hypotheses $\mathcal{O} \subseteq \{o \mid o : \mathcal{X} \to \mathcal{Y}\}$, this class contains hypotheses characteristic to disagreements in $\mathcal{O}$:
\begin{equation}\footnotesize
    \mathcal{O} \Delta \mathcal{O} \overset{\mathrm{def}}{=} \{x \mapsto \mathtt{NC}[o(x) \neq o'(x)] \mid o,o' \in \mathcal{O}\}
\end{equation}
where $\mathtt{NC}[\cdot]$ acts like an indicator function, returning $\mathtt{NC}$ for true arguments and $\mathtt{CP}$ otherwise.
Using this class, we identify a relationship between the true error when identifying non-cooperation $\mathbf{cer}_\theta$ and the observed object-identification errors $\widehat{\mathbf{oer}}_{S}(\cdot | \ \mathtt{CP})$ and $\widehat{\mathbf{oer}}_{S}(\cdot | \ \mathtt{NC})$ against the cooperative and non-cooperative answer-player, respectively. While a more traditional learning-theoretic bound would relate $\mathbf{cer}_\theta$ to the empirical observation $\widehat{\mathbf{cer}}_S$ for the same task, our novel bound reveals a connection to the seemingly distinct task of object-identification. Later, this relationship is useful for analyzing how the question-player's communication policy controls the data-distribution so that \textit{both} objectives are improved. Proofs of all result are provided in Section~\ref{sec:proofs}.
\begin{theorem}
\label{thm:main}
Define $\mathcal{O}$ as above and take $\mathcal{C}$ to be sufficiently complex so that $\mathcal{O}\Delta\mathcal{O} \subseteq \mathcal{C}$. Let $d$ be the VC-Dimension of $\mathcal{C}$. Then for any $\delta \in (0,1)$, with probability at least $1-\delta$, for all $o,o' \in \mathcal{O}$,
\begin{equation}\footnotesize
    \mathbf{cer}_\theta(\hat{c}) \leq \widehat{\mathbf{p}}_S + \widehat{\mathbf{oer}}_S(o) - \Delta_S(o') + C
\end{equation}
where $C = (4 +   \sqrt{d \log (2em/d)}) / (\delta \sqrt{2m})$, $S \overset{\mathrm{iid}}{\sim} \mathbb{P}_\theta$, and $\hat{c} \in \argmin_{c \in \mathcal{C}} \widehat{\mathbf{cer}}_{S}(c)$.
\end{theorem}
\paragraph{Remarks.} Notice, one sensible choice of $o$ and $o'$ is to pick $o$ which minimizes the observed object-identification error and $o'$ which maximizes $\Delta_S$; this produces the tightest bound on the expected cooperation-identification error. We leave these hypotheses unspecified because later we must make limiting assumptions on the properties of $o$ and $o'$ (e.g., Prop.~\ref{prop:method_bound}). Greater generality here makes our results more broadly applicable. Besides this, we also observe that $C$ goes to 0 as $m$ grows. Ultimately, we ignore $C$ in interpretation, but point out that bounds based on the VC-Dimension (as above) are notoriously loose for most $\mathbb{P}_\theta$. As we are primarily interested in these bounds for purpose of interpretation and algorithm design, this is a non-issue. On the other hand, if practically computable bounds are desired, other (more data-dependent) techniques may be fruitful; e.g., see \citet{dziugaite2017computing}. 
\paragraph{Interpretation.}
As noted, the question-player has some control over the distribution $\mathbb{P}_\theta$ through the communication policy $\pi_\theta$. So, Thm.~\ref{thm:main} can be interpreted to motivate indirect mechanisms for controlling the cooperation-identification error $\mathbf{cer}_\theta(\hat{g})$. Specifically, with respect to $\widehat{\mathbf{oer}}_S$, we can infer that improving performance on the object identification task should implicitly improve performance on the separate task of identifying non-cooperation. The term $\Delta_S$ also offers insight. It suggests certain non-cooperative answer-players -- whose actions induce a large reduction in performance as compared to the cooperative answer-player -- are easy to identify. Stated more plainly, non-cooperative agents reveal themselves by their non-cooperation; this is true, in particular, when their behavior causes large performance drops. In Section~\ref{sec:comm_analysis}, we formalize these concepts further.
\subsection{Analyzing Communication Strategies} 
\label{sec:comm_analysis}
In this section, we analyze methods for the question-player to select the communication policy $\pi_\theta$. In recent dialogue literature, reinforcement learning (RL) has proven successful in teaching agents effective communication strategies. For example, \citet{strub2017end} show this to be the case in the fully cooperative version of \textit{Guess What?!}. Selecting an appropriate reward structure is fundamental to any RL training regime. To this end, we use Thm~\ref{thm:main} to study different reward structures. We consider, in particular, an episodic RL scenario where the discount factor (often called $\gamma$) is set to 1 and the only non-zero reward comes at the end of the episode. So, the question-player holds a full dialogue with the answer-player, guesses the goal-object and answer-player's cooperation based on this dialogue, and then receives a reward dependent on whether the guesses are correct.  Under these assumptions, the question-player selects the communication policy $\pi_\theta$ to maximize:
\begin{equation}\footnotesize
\label{eqn:J}
    J(\theta) = \mathbf{E} \left [ \ \rho(X, Y, Z) \ \right]; \quad (X,Y,Z) \sim \mathbb{P}_\theta
\end{equation}
where $\rho: \mathcal{X}\times\mathcal{Y}\times\mathcal{Z} \to \mathbb{R}$ is the reward structure to be decided. In particular, selection of $\theta$ can often be achieved through policy gradient methods. \citet{williams1992simple} and \citet{sutton1999policy} are attributed with showing we can estimate $\nabla_\theta J(\theta)$ in an un-biased manner through Monte-Carlo estimation. In our implementation in Section~\ref{sec:results}, our particular policy gradient technique is identical to previous work on communication strategies for the \textit{Guess What?!} dataset \cite{strub2017end}. Thus, we focus discussion on the reward structure $\rho$ and understanding its role through a theoretical lens.

To select $\rho$, we first consider some obvious choices without appealing to complex analysis. Specifically, for $c$ fixed, define $\rho(X,Y,Z) = \mathbf{1}[c(X) = Z]$. Then,
\begin{equation}\footnotesize
\label{eqn:rew_naive}
    J(\theta) = 1 - \mathbf{cer}_\theta(c).
\end{equation}
Thus, maximizing $J(\theta)$ is equivalent to minimizing the cooperation-identification error. This reward focuses \textit{only} on identifying non-cooperation. On the other hand, if $\rho(X,Y,Z) = \mathbf{1}[o(X) = Y]$ for some fixed $o$, then
\begin{equation}\footnotesize
\label{eqn:rew_ours}
    J(\theta) = 1 - \mathbf{oer}_\theta(o)
\end{equation}
So, in this case, maximizing $J(\theta)$ minimizes the expected object-identification error.

It is easy to see the trade-off between the two choices discussed above. Each focuses \textit{distinctly} on a single objective of the question-player and it is not clear how these two objectives can relate to each other. To properly answer this, we appeal to analysis.
We first give some definitions.
\begin{defn}
\label{defn:other_def}
We say a hypothesis $o \in \mathcal{O}$ is $\alpha$-improved by $\theta^*$ relative to $\theta$ if $J(\theta^*) \geq J(\theta) + \alpha$ for $\rho(X,Y,Z) = \mathbf{1}[o(X) = Y]$ and $\alpha \geq 0$.
\end{defn}
Simply, Def.~\ref{defn:other_def} formally describes when a communication policy
$\pi_{\theta^*}$ improves the question-player's ability to identify the goal-object. Next, we define efficacy of an answer-player as a property of the errors induced by this player's dialogue.
\begin{defn}
\label{defn:only_def}
We say a non-cooperative answer-player is effective with fixed parameter $\epsilon$ if for all $\delta > 0$ there is $n$ such that for all $\theta, \theta' \in \Theta$, $o \in \mathcal{O}$, and $m \geq n$, we have
\begin{equation}\footnotesize
\mathbf{Pr}( \ | \widehat{\mathbf{oer}}_T(o \mid \mathtt{NC}) - \widehat{\mathbf{oer}}_{S}(o \mid \mathtt{NC}) | > \epsilon) \leq \delta
\end{equation}
where $S \overset{\mathrm{iid}}{\sim} \mathbb{P}_\theta$, $T \overset{\mathrm{iid}}{\sim} \mathbb{P}_{\theta'}$.
\end{defn}
 Def.~\ref{defn:only_def} requires that the error of all question-players converge in probability to the same $O(\epsilon)$-sized region when playing against an effective answer-player. If a non-cooperative answer-player is effective, then regardless of the communication strategy employed by the question-player, we should not expect to observe large changes in object-identification performance against the non-cooperative opponent. Conceptually, this captures the following idea: \textit{Without cooperation, we cannot expect interlocutors to make significant headway.} This assumption is inherently related to an answer-player's failure to abide by Gricean maxims of conversation: uninformative and deceitful  responses violate the maxim of relation and quality, respectively. Instead of explicitly modelling these violations,  Def.~\ref{defn:only_def} focuses on the \textit{effect} of violations -- namely, failure to progress. While violation of other Gricean maxims (i.e., quantity and manner) are less applicable to the simple game we consider, the definition of non-cooperation we give (as an observable effect) still applies.

As alluded, when the non-cooperative answer-player is effective, this non-cooperation is enough to reveal the answer-player to the question-player. The question-player may focus on communicating to identify the goal-object and this will reduce all terms in the upper-bound of Thm.~\ref{thm:main}; subsequently, we expect this communication strategy to be effective not only for identifying the goal-object, but also for identifying non-cooperation. 
\begin{prop}
\label{prop:method_bound}
Let $o,o' \in \mathcal{O}$ and $\theta^*, \theta \in \Theta$. Suppose the non-cooperative answer-player is effective and further suppose both $o$ and $o'$ are $\alpha$-improved by $\theta^*$ relative to $\theta$ with $\alpha > \epsilon$. Then, for any $\delta > 0$, there is $n$ such that for all $m \geq n$, with probability at least $1 - \delta - \gamma$ we have
\begin{equation}\footnotesize
\begin{split}
    & \widehat{\mathbf{p}}_T + \widehat{\mathbf{oer}}_T(o) - \Delta_T(o') \\
    \leq \quad & \widehat{\mathbf{p}}_S + \widehat{\mathbf{oer}}_S(o) - \Delta_S(o') + O(C)
\end{split}
\end{equation}
where $S \overset{\mathrm{iid}}{\sim} \mathbb{P}_\theta$, $T \overset{\mathrm{iid}}{\sim} \mathbb{P}_{\theta^*}$, $\gamma = 2\exp \left( -m \omega^2 / 2 \right)$, $\omega = \alpha - \epsilon$, and
$C = (2m)^{-\frac{1}{2}}\sqrt{\ln 6 - \ln \delta}$.
\end{prop}
\paragraph{Remarks.} Notice, the result assumes the hypotheses $o, o'$ and policies $\pi_\theta, \pi_{\theta^*}$ are fixed \textit{a priori} to drawing $S, T$. Hence, the bound is only valid for test sets independent from training. Regardless, it is still useful for interpretation and this style of bound produces tighter guarantees than conventional learning-theoretic bounds; i.e., from both analytic and empirical perspectives, respectively \citep{shalev2014understanding, sicilia2021}. 
Like Thm.~\ref{thm:main}, we also use two hypotheses $o, o' \in \mathcal{O}$, but the result is easily specified to the one hypothesis case
by taking $o=o'$ (albeit, this may loosen the bound). In any case, the assumption is not unreasonable. A policy $\pi_{\theta^*}$ -- optimized with respect to just one hypothesis $o$ -- may also offer relative improvement for other hypotheses distinct from $o$. For greater certainty, the term $\delta$ in the probability can be made arbitrarily small provided a large enough sample. Sensibly, the term $\gamma$ indicates the probability is also proportional to how much better the communication mechanism $\pi_{\theta^*}$ is where ``better'' is given precise meaning by comparing population statistics for the objective $J(\cdot)$ via $\alpha$. At minimum, we require $\alpha > \epsilon$, but $\epsilon$ should be small for suitably effective answer-players anyway. Finally, we again, safely ignore $O(C)$ terms, which go to 0 as $m$ grows.
\paragraph{Interpretation.}
The takeaway from Prop.~\ref{prop:method_bound} is an unexpectedly sensible strategy for game success: the question-player focuses communication efforts \textit{only} on identifying the goal-object. When the non-cooperative agent is effective, this communication strategy essentially reduces an upperbound on the true cooperation-identification error. All the while, this strategy very obviously assists the object-recognition task as well. We again note the implication that non-cooperative agents can reveal themselves by their non-cooperation. The question-player need not expend additional effort to uncover them by dialogue actions.
\paragraph{Comparison to Thm.~\ref{thm:main}.}
While Thm.~\ref{thm:main} alludes the interpretation given above -- since the object-identification error is shown to control cooperation identification error in part --, Prop.~\ref{prop:method_bound} distinguishes itself because it considers \textit{all} terms in the upperbound (not just $\widehat{\mathbf{oer}}$). This subtlety is important. In particular, a priori, one cannot be certain that improving the object-identification error from $S$ to $T$ \textit{also} improves the cooperation gap~$\Delta$. Instead, it could be the case that $\Delta$ decreases and the overall bound on $\mathbf{cer}$ is worsened. Aptly, Prop.~\ref{prop:method_bound} isolates the circumstances (i.e., related to Def.~\ref{defn:only_def}), which ensure this adverse effect does not occur. It shows us, under reasonable assumptions, the communication strategy discussed in our interpretation controls the \textit{whole} bound in Thm.~\ref{thm:main} and not just some part. As noted, drawing inference from only a portion of the bound can have unexpected consequences. In fact, this is the topic of much recent work in analysis of learning algorithms \citep{johansson2019support, wu2019domain, zhao2019learning, sicilia2022pac}. 
\paragraph{Comparison to Cooperative Setting.} It is also worthwhile to note that setting the reward as $\rho(X,Y,Z) = \mathbf{1}[o(X) = Y]$ is also an appropriate strategy in the distinct \textit{fully} cooperative \textit{Guess What?!} game. The authors of the original \textit{GuessWhat?!} corpus propose this reward exactly in their follow-up work \cite{strub2017end} which uses RL to learn communication strategies in the fully cooperative setting. Thus, the theoretical results of this section are exceedingly practical. They suggest, for effective non-cooperative agents, we may sensibly employ the same techniques in both the fully cooperative setting and the partially non-cooperative setting. This is beneficial, because the nature of our problem anticipates we will not know the setting in which we operate.
\paragraph{Motivating a Mixed Objective.} As a final note, we remark on how this result may be applied to properly motivate a reward which, \textit{a priori}, can only be heuristically justified. Specifically, a very reasonable suggestion 
would be to combine the rewards in  Eq.~\eqref{eqn:rew_naive} and Eq.~\eqref{eqn:rew_ours} via convex sum. Prior to our theoretical analyses, it is unclear that the two strategies would be complementary. Instead, the objectives could be competing, and so, this mixed strategy could lead to sub-par performance on both tasks. In light of this, our theoretical results help to understand this heuristic more formally. They suggest the two strategies are, in fact, complementary and outline the assumptions necessary for this to be the case. In contrast, empirical analyses can be much more specific to the data used, among other factors. This, in general, is a key differentiation between the analysis we have provided here and the oft-used appeal to heuristics.
\subsection{Proofs}
\label{sec:proofs}
Here, we provide proof of all theoretical results. We first remind the reader of some key definitions for easy reference:
\begin{equation}\scriptsize
    \begin{split}
        & \mathbf{Pr}(Z = \mathtt{NC}) = \mathbf{p}_\mathtt{NC}; \quad (X,Y,Z) \sim \mathbb{P}_\theta; \\
        & \widehat{\mathbf{p}}_S \overset{\mathrm{def}}{=} \tfrac{1}{m}\sum\nolimits_i \mathbf{1}[Z_i = \mathtt{NC}]; \\
        & \widehat{\mathbf{oer}}_S(o) \overset{\mathrm{def}}{=} \tfrac{1}{m}\sum\nolimits_{i=1}^m \mathbf{1}[o(X_i) \neq Y_i]; \\
        & \widehat{\mathbf{cer}}_S(c) \overset{\mathrm{def}}{=} \tfrac{1}{m}\sum\nolimits_{i=1}^m \mathbf{1}[c(X_i) \neq Z_i]; \\
        & \widehat{\mathbf{oer}}_{S}(o \mid \mathtt{CP}) \overset{\mathrm{def}}{=} \widehat{\mathbf{oer}}_{S'}(o), \ \ S' = ((X_i, Y_i) \mid Z_i = \mathtt{CP}); \\
        & \Delta_S(o) \overset{\mathrm{def}}{=} \widehat{\mathbf{p}}_S \cdot \widehat{\mathbf{oer}}_{S}(o | \mathtt{NC}) - (1 - \widehat{\mathbf{p}}_S)  \cdot \widehat{\mathbf{oer}}_S(o | \mathtt{CP}).
    \end{split}
\end{equation}
Please, see Section~\ref{sec:theory_setup} for additional definitions and context.
\paragraph{Theorem~\ref{thm:main}.}
\begin{claim}
Define $\mathcal{O}$ as above and take $\mathcal{C}$ to be sufficiently complex so that $\mathcal{O}\Delta\mathcal{O} \subseteq \mathcal{C}$. Let $d$ be the VC-Dimension of $\mathcal{C}$. Then for any $\delta \in (0,1)$, with probability at least $1-\delta$, for all $o,o' \in \mathcal{O}$,
\begin{equation}\footnotesize
    \mathbf{cer}_\theta(\hat{c}) \leq \widehat{\mathbf{p}}_S + \widehat{\mathbf{oer}}_S(o) - \Delta_S(o') + C
\end{equation}
where $C = (4 +   \sqrt{d \log (2em/d)}) / (\delta \sqrt{2m})$, $S \overset{\mathrm{iid}}{\sim} \mathbb{P}_\theta$, and $\hat{c} \in \argmin_{c \in \mathcal{C}} \widehat{\mathbf{cer}}_{S}(c)$.
\end{claim}
\begin{proof}
For any $c \in \mathcal{C}$ and $\delta \in (0,1)$, we have
\begin{equation}\footnotesize
    \mathbf{Pr}\left ( \mathbf{cer}_\theta(c) \leq \widehat{\mathbf{cer}}_{S}(c) + C\right) \geq 1 - \delta.
\end{equation}
This is a standard VC-bound; e.g., Thm.~6.11 in \citet{shalev2014understanding}. Thus, it suffices to show that for any sample $S$ of size $m$ and any choice of hypotheses $o, o' \in \mathcal{H}$, we have
\begin{equation}\footnotesize
    \widehat{\mathbf{cer}}_{S}(\hat{c}) \leq \widehat{\mathbf{p}}_S + \widehat{\mathbf{oer}}_S(o) - \Delta_S(o').
\end{equation}
Notice first, by choice of $\hat{c}$, for any $c \in \mathcal{C}$ we have 
\begin{equation}\footnotesize
    \widehat{\mathbf{cer}}_{S}(\hat{c}) \leq \widehat{\mathbf{cer}}_{S}(c).
\end{equation}
By definition of $\mathcal{O}\Delta\mathcal{O}$ and its relation to $\mathcal{C}$, for any choice of $o,o' \in \mathcal{O}$, there is some $c' \in \mathcal{C}$ such that $c'(X) = \mathtt{NC}[o(X) \neq o'(X)]$ for all $X$. Recall, $\mathtt{NC}[\cdot]$ acts like an indicator function, returning $\mathtt{NC}$ for true arguments and $\mathtt{CP}$ otherwise. Thus, 
\begin{equation}\footnotesize
\label{eqn:bounding_mer}
\begin{split}
   & \widehat{\mathbf{cer}}_{S}(\hat{c}) \leq \widehat{\mathbf{cer}}_{S}(c') \\
   & = \widehat{\mathbf{p}}_S - \frac{1}{m} \sum_{i \in \{k \mid Z_k = \mathtt{NC}\}} \mathbf{1}[o(X_i) \neq o'(X_i)] \\
   & \qquad + \frac{1}{m} \sum_{j \in \{k \mid Z_k = \mathtt{CP}\}} \mathbf{1}[o(X_j) \neq o'(X_j)].
\end{split}
\end{equation}
The equality follows by applying the definition of $c'$, appropriately grouping terms, and then using the fact: $\mathbf{1}[o(X_i) = o'(X_i)] = 1 - \mathbf{1}[o(X_i) \neq o'(X_i)]$. Now, the triangle inequality for classification error \cite{crammer2007learning, ben2007analysis} tells us for any $(X,Y) \in \mathcal{X} \times \mathcal{Y}$ and any $o, o' \in \mathcal{O}$ we have
\begin{equation}\footnotesize
\begin{split}
& \mathbf{1}[o'(X) \neq Y] - \mathbf{1}[o(X) \neq Y] \leq \mathbf{1}[o(X) \neq o'(X)] \\
& \leq  \mathbf{1}[o(X) \neq Y] + \mathbf{1}[o'(X) \neq Y].
\end{split}
\end{equation}
Applying these bounds to the result of  Eqn.~\eqref{eqn:bounding_mer} and re-arranging terms completes the proof.
\end{proof}

\paragraph{Proposition~\ref{prop:method_bound}.}
\begin{claim}
Let $o,o' \in \mathcal{O}$ and $\theta^*, \theta \in \Theta$. Suppose the non-cooperative answer-player is effective and further suppose both $o$ and $o'$ are $\alpha$-improved by $\theta^*$ relative to $\theta$ with $\alpha > \epsilon$. Then, for any $\delta > 0$, there is $n$ such that for all $m \geq n$, with probability at least $1 - \delta - \gamma$ we have
\begin{equation}\footnotesize
\begin{split}
    & \widehat{\mathbf{p}}_T + \widehat{\mathbf{oer}}_T(o) - \Delta_T(o') \\
    \leq \quad & \widehat{\mathbf{p}}_S + \widehat{\mathbf{oer}}_S(o) - \Delta_S(o') + O(C)
\end{split}
\end{equation}
where $S \overset{\mathrm{iid}}{\sim} \mathbb{P}_\theta$, $T \overset{\mathrm{iid}}{\sim} \mathbb{P}_{\theta^*}$, $\gamma = 2\exp \left( -m \omega^2 / 2 \right)$, $\omega = \alpha - \epsilon$, and
$C = (2m)^{-\frac{1}{2}}\sqrt{\ln 6 - \ln \delta}$.
\end{claim}
We first give a Lemma.
\begin{lemma}
\label{lemma:rl_eta}
Let $o\in \mathcal{O}$ and $\theta, \theta^* \in \Theta$. For any $\epsilon \geq 0$, suppose $o$ is $\alpha$-improved by $\theta^*$ relative to $\theta$ with $\alpha > \epsilon$. Then, 
\begin{equation}\footnotesize
    \mathbf{Pr} \left (  \widehat{\mathbf{oer}}_{T}(o) \geq \widehat{\mathbf{oer}}_{S}(o) - \epsilon \right) \leq \exp ( -\tfrac{m}{2}(\alpha - \epsilon)^2 )
\end{equation}
where $S \overset{\mathrm{iid}}{\sim} \mathbb{P}_\theta$, $T \overset{\mathrm{iid}}{\sim} \mathbb{P}_{\theta^*}$.
\end{lemma}
\begin{proof}
Given samples $S \overset{\mathrm{iid}}{\sim} \mathbb{P}_\theta$ and $T \overset{\mathrm{iid}}{\sim} \mathbb{P}_{\theta^*}$ with $S=(X_i, Y_i, Z_i)_i$ and $T=(X_i^*, Y_i^*, Z_i^*)_i$ define
\begin{equation}\footnotesize
    U = \frac{1}{m}\sum_{i=1}^m \rho(X_i, Y_i, Z_i) - \rho(X^*_i, Y^*_i, Z^*_i).
\end{equation}
Then, $\mathbf{E}[U] = J(\theta) - J(\theta^*)$ and application of Hoeffding's inequality yields
\begin{equation}\footnotesize
\begin{split}
    & \mathbf{Pr} (U \geq -\epsilon ) \leq \exp\left (-\tfrac{m}{2}(J(\theta^*) - J(\theta) - \epsilon)^2\right )
\end{split}
\end{equation}
To finish, apply $J(\theta^*) - J(\theta) - \epsilon \geq \alpha - \epsilon > 0$.
\end{proof}
Now, we proceed with the proof of Prop.~\ref{prop:method_bound}.
\begin{proof}
We begin by bounding the probability of a few events of interest. First,
\begin{equation}\footnotesize
    \mathbf{Pr}(\widehat{\mathbf{oer}}_{T}(o) \geq \widehat{\mathbf{oer}}_{S}(o) - \epsilon ) \leq \frac{\gamma}{2}
\end{equation}
as well as 
\begin{equation}\footnotesize
    \mathbf{Pr} \left (  \widehat{\mathbf{oer}}_{T}(o') \geq \widehat{\mathbf{oer}}_{S}(o') - \epsilon \right) \leq \frac{\gamma}{2}
\end{equation}
by two applications of Lemma~\ref{lemma:rl_eta}. 
Second, by Hoeffding's Inequality, for any $\delta \in (0,1)$ we know with $C = (2m)^{-\frac{1}{2}}\sqrt{\ln 6 - \ln \delta}$
\begin{equation}\footnotesize
\label{eqn:a1}
    \mathbf{Pr}\left (|\widehat{\mathbf{p}}_T - \mathbf{p}_\mathtt{NC} | \geq C\right) \leq \frac{\delta}{3}
\end{equation}
and
\begin{equation}\footnotesize
\label{eqn:a2}
    \mathbf{Pr}\left (|\widehat{\mathbf{p}}_S - \mathbf{p}_\mathtt{NC} | \geq C\right) \leq \frac{\delta}{3}.
\end{equation}
Third, by assumption on the non-cooperative agent, we know we may pick large enough samples $S$ and $T$ so  
\begin{equation}\footnotesize
    \mathbf{Pr}(| \widehat{\mathbf{oer}}_T(o \mid \mathtt{NC}) - \widehat{\mathbf{oer}}_{S}(o \mid \mathtt{NC}) | > \epsilon) \leq \frac{\delta}{3}.
\end{equation}
Applying Boole's inequality bounds the probability that any one of these events holds by $\delta+ \gamma$. Considering the complement event yields a lower bound on the probability that every one of these events fails to hold. Specifically, the lower bound is $1 - \delta - \gamma$. Thus, it is sufficient to show
\begin{equation}\footnotesize
\begin{split}
    & \widehat{\mathbf{p}}_T + \widehat{\mathbf{oer}}_T(o) - \Delta_T(o') \\
    \leq \quad & \widehat{\mathbf{p}}_S + \widehat{\mathbf{oer}}_S(o) - \Delta_S(o') + O(C)
\end{split}
\end{equation}
under assumption of the complement event. To this end, assume the complement. Then, we have directly that
\begin{equation}\footnotesize
    \widehat{\mathbf{p}}_T + \widehat{\mathbf{oer}}_T(o) \leq \widehat{\mathbf{p}}_S + \widehat{\mathbf{oer}}_S(o) + 2C - \epsilon
\end{equation}
So, in the remainder, we concern ourselves with showing $-\Delta_T(o') \leq -\Delta_S(o') + \epsilon + O(C)$. First note that for $T$ it is always true that
\begin{equation}\footnotesize
\begin{split}
    \widehat{\mathbf{oer}}_{T}(o') = \ & \widehat{\mathbf{p}}_T\cdot \widehat{\mathbf{oer}}_{T}(o' | \mathtt{NC}) \\
    & + (1 - \widehat{\mathbf{p}}_T)\cdot \widehat{\mathbf{oer}}_{T}(o' | \mathtt{CP}).
\end{split}
\end{equation}
A similar equation holds for $S$. Then, $\widehat{\mathbf{oer}}_{T}(o') \leq \widehat{\mathbf{oer}}_{S}(o') - \epsilon$ by assumption, so expanding,
\begin{equation}\footnotesize
\label{eqn:rearange}
    \begin{split}
         & (1 - \widehat{\mathbf{p}}_T)\cdot \widehat{\mathbf{oer}}_{T}(o' | \mathtt{CP}) - \widehat{\mathbf{p}}_S\cdot \widehat{\mathbf{oer}}_{S}(o' | \mathtt{NC})\\
        & \leq (1 - \widehat{\mathbf{p}}_S) \widehat{\mathbf{oer}}_{S}(o' | \mathtt{CP}) - \widehat{\mathbf{p}}_T \widehat{\mathbf{oer}}_{T}(o' | \mathtt{NC}) - \epsilon.
    \end{split}
\end{equation}
We also assume $|\widehat{\mathbf{p}}_S - \mathbf{p}_\mathtt{NC}| \leq C$ and $|\mathbf{p}_\mathtt{NC} - \widehat{\mathbf{p}}_T| \leq C$, so applying to both sides of Eq.~\eqref{eqn:rearange} yields
\begin{equation}\footnotesize
\label{eqn:lam_exch}
\begin{split}
    & (1 - \widehat{\mathbf{p}}_T) \cdot \widehat{\mathbf{oer}}_T(o' | \mathtt{CP}) - \widehat{\mathbf{p}}_T \cdot \widehat{\mathbf{oer}}_S(o' | \mathtt{NC}) \\
    & \leq (1 - \widehat{\mathbf{p}}_S)\cdot \widehat{\mathbf{oer}}_S(o' | \mathtt{CP} ) - \widehat{\mathbf{p}}_S \cdot \widehat{\mathbf{oer}}_T(o' | \mathtt{NC} ) \\
    & \qquad - \epsilon + 2C\cdot (\widehat{\mathbf{oer}}_T(o' | \mathtt{NC} ) + \widehat{\mathbf{oer}}_S(o' | \mathtt{NC})) \\
    & \leq (1 - \widehat{\mathbf{p}}_S)\cdot \widehat{\mathbf{oer}}_S(o' | \mathtt{CP} ) - \widehat{\mathbf{p}}_S\cdot \widehat{\mathbf{oer}}_T(o' | \mathtt{NC} ) \\
    & \qquad - \epsilon + 4C
\end{split}
\end{equation}
Finally, the fact $| \widehat{\mathbf{oer}}_S(o' | \mathtt{NC} ) - \widehat{\mathbf{oer}}_T(o' | \mathtt{NC} )| \leq \epsilon$ may be applied to both sides of Eq.~\eqref{eqn:lam_exch} to attain
\begin{equation}\footnotesize
\begin{split}
    & (1 - \widehat{\mathbf{p}}_T) \cdot \widehat{\mathbf{oer}}_T(o' | \mathtt{CP}) - \widehat{\mathbf{p}}_T \cdot \widehat{\mathbf{oer}}_T(o' | \mathtt{NC}) \\
    & \leq (1 - \widehat{\mathbf{p}}_S)\cdot \widehat{\mathbf{oer}}_S(o' | \mathtt{CP} ) - \widehat{\mathbf{p}}_S\cdot \widehat{\mathbf{oer}}_S(o' | \mathtt{NC}) \\
    & \qquad - \epsilon + 4C + (\widehat{\mathbf{p}}_S + \widehat{\mathbf{p}}_T)\epsilon \\
    & \leq (1 - \widehat{\mathbf{p}}_S)\cdot \widehat{\mathbf{oer}}_S(o' | \mathtt{CP} ) - \widehat{\mathbf{p}}_S\cdot \widehat{\mathbf{oer}}_S(o' | \mathtt{NC}) \\
    & \qquad + 4C + \epsilon.
\end{split}
\end{equation}
\end{proof}

\begin{figure*}[ht]
    \centering
    \includegraphics[width=.99\textwidth]{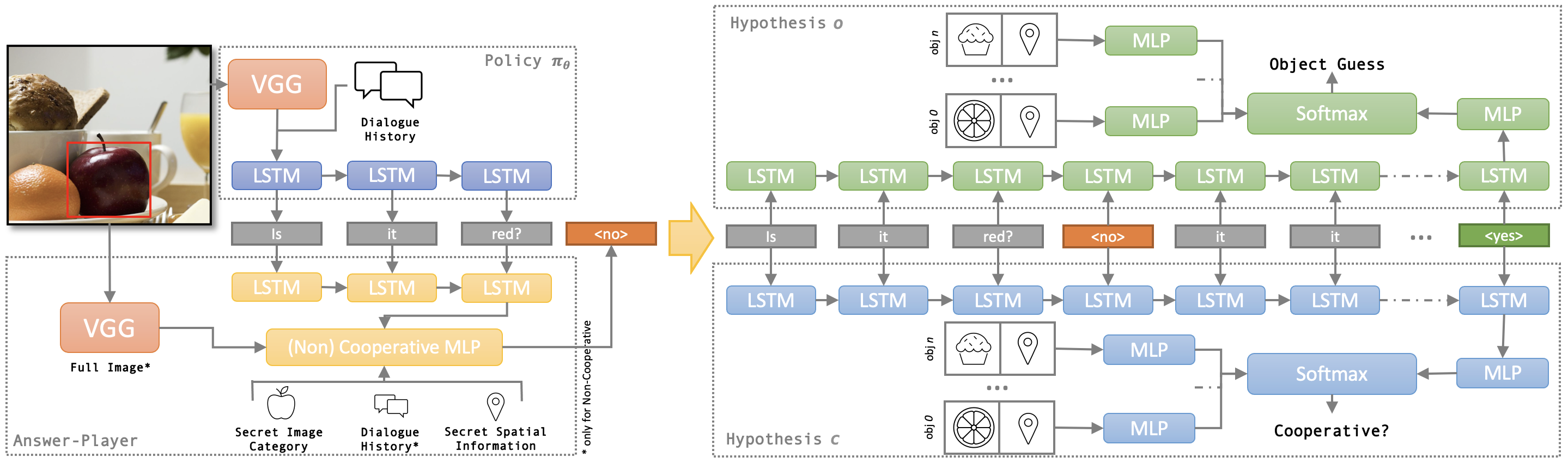}
    \caption{\small{Architecture used in our implementation. 
    Object categories and words are represented using one-hot encoding so an embedding is learned for each object/word. Locations are represented by assigning a common coordinate-system to all images and reporting the object center's image-relative coordinates.}}
    \label{fig:full_architecture-main-text}
\end{figure*}
\section{Experimentation}
\label{sec:results}
In this section, we empirically study the communication strategies just discussed in a theoretical context. We also give insights on the non-cooperative strategies found in the collected data. 

\subsection{Implementation}
Our implementation makes use of the existing framework of \citet{de2017guesswhat}. The primary difference in the game we consider is the included possibility that the answer-player is non-cooperative. As such, many of our model components are based on those proposed by the dataset authors \citep{de2017guesswhat, strub2017end}. 

\paragraph{Question-Player.}
The question-player consists of: a hypothesis $o$ which predicts the goal-object given the object categories,
object locations,
and the dialogue-history;
a hypothesis $c$ which predicts cooperation given the same information; and the communication policy $\pi_\theta$ which generates dialogue given the image\footnote{The image is processed by a VGG network and these features initialize the LSTM state in Figure~\ref{fig:full_architecture-main-text}.} and the current dialogue-history. Each is modelled by a neural-network. Architectures of $o$ and the policy $\pi_\theta$ are identical to the \textit{guesser} model and \textit{questioner} model described by \citet{strub2017end}. We give an overview of the architectures in Figure~\ref{fig:full_architecture-main-text} as well.
\paragraph{Answer-Player.}
The cooperative answer-player is modeled by a neural-network with binary output dependent only on the goal-object and the most immediate question. \citet{strub2017end} demonstrate -- in the cooperative case -- that additional features do not improve performance. On the other hand, non-cooperative behaviors may require more complex modeling. We explore different features for the network modeling the non-cooperative answer-player. During experimentation, we condition on various combinations of the full (and immediate) dialogue-history, the image, and the goal-object. The architectures in both cases are based on the \textit{oracle} model described by \citet{strub2017end} with the addition of an LSTM that allows conditioning on the full dialogue-history. See Figure~\ref{fig:full_architecture-main-text} for an overview.
\paragraph{Training.}
As noted, $o$ is assumed fixed before considering the task of $c$. In practice, we achieve this through supervised learning (SL) by training $o$ on human games in the \textit{Guess What?!} (GW) corpus. Similarly, the cooperative answer-player is trained via SL on the GW corpus. The non-cooperative answer-player uses our novel corpus of non-cooperative games (see Section~\ref{sec:data}). Following \citet{strub2017end}, we pre-train the communication policy $\pi_\theta$ using SL on the GW corpus. In some cases, $\pi_\theta$ is then taught a specific communication strategy by fine-tuning with RL on simulated dialogue. Dialogue is simulated by randomly sampling $Z \sim \mathrm{Bernoulli}(\mathbf{p}_\mathtt{NC})$, drawing an image-object pair uniformly at random from the GW corpus, and allowing the current policy $\pi_\theta$ and the already trained answer-player indicated by $Z$ to converse 5 rounds. The hypothesis $c$ is trained simultaneously on simulated dialogue during the RL phase of $\pi_\theta$ via SL. We do so because $c$ is assumed to minimize sample error in Thm~\ref{thm:main}. While simultaneous gradient methods only approximate this goal, it is more in line with assumptions than fixing $c$ a priori. In general, hyper-parameters are fixed for all experiments and are detailed in the code, which is \href{https://github.com/anthonysicilia/modeling-non-cooperation-tacl2022}{publicly available}. When possible, we follow the parameter choices of \citet{strub2017end}. As an exception, we shorten the number of epochs in the RL phase to 10. Recall, the new network $c$ is trained in this phase as well. For $c$, the learning rate is 1e-4. The new non-cooperative answer-players are trained similarly to the cooperative answer-players -- i.e., as in \citet{strub2017end} -- but we remove early-stopping to avoid the need for a validation set. Our non-cooperative corpus is thus used in its entirety for training since all trained agents are evaluated on novel generated dialogue (see Section~\ref{sec:results_results}). When training with the GW corpus, we use the original train/val split.
\paragraph{Comparison.}
Despite some slight deviations from the original \textit{Guess What?!} training setup, we point out that our fully cooperative results are fairly similar. In Figure~\ref{fig:main-result}, we show error-rate on simulated, cooperative, test dialogues for our question-player trained solely on object-identification; the precise error-rate is 48.8\%. For the most similar training and testing setup used by \citet{strub2017end}, the question-player achieves an error-rate of 46.7\%. 
\subsection{Results}
\label{sec:results_results}
\begin{figure*}
    \centering
    \includegraphics[width=.99\textwidth]{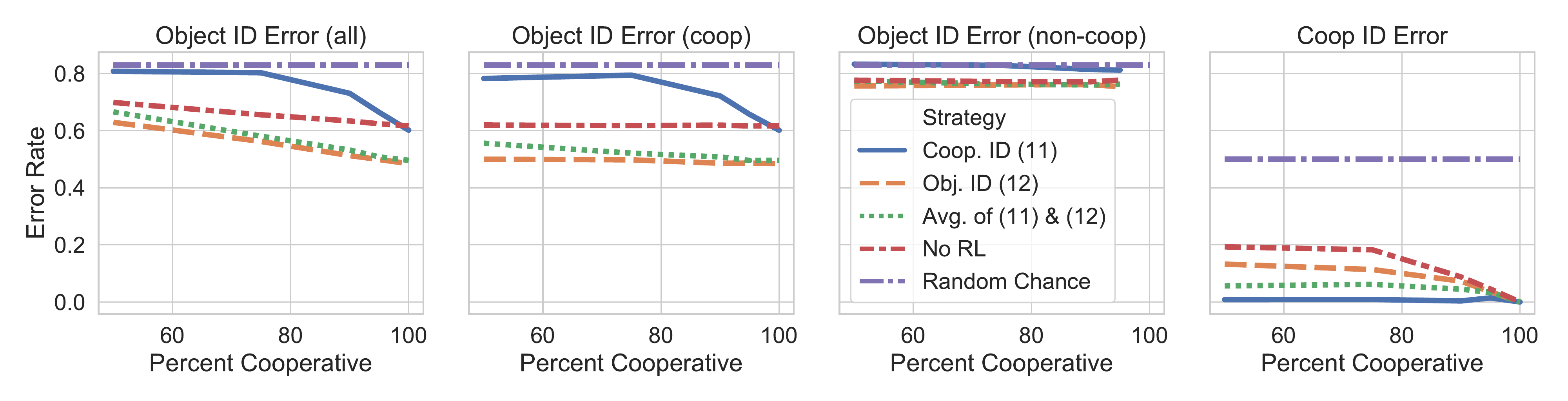}
    \caption{\small{The first three communication strategies (top to bottom in the legend) correspond to using RL with the objective described by Eq.~\eqref{eqn:rew_naive}, Eq.~\eqref{eqn:rew_ours}, or an average of both. Respectively, the last two strategies correspond to using no RL to learn a strategy (i.e., supervised learning only) or to making predictions at random. For object-identification error, parentheses indicate the subset of examples on which the error rate is computed. For non-cooperation detection, the error rate is computed on all samples. Overall, results validate our theoretical argument.}}
    \label{fig:main-result}
\end{figure*}
We report error for cooperation-identification and object-identification. We use a sample $S$ which has simulated dialogue (see \textbf{Training}) between our trained question- and answer-players using about 23K image-object pairs sampled from the GW test set. The objects/images are fixed for all experiments, but dialogue will of course change depending on the question-player. Each data-point in the figures corresponds to a single run using a specified percentage of cooperative examples; i.e., the answer-player's type is selected by sampling $\mathrm{Bernoulli}(\mathbf{p}_\mathtt{NC})$ and setting $\mathbf{p}_\mathtt{NC}$ as the desired $\%$.
\paragraph{Human Non-Cooperative Strategies.}
Between qualitative analysis of this data and conversations with the workers, we determined three primary human strategies for deception: \textit{spamming}, \textit{absolute contradiction}, and \textit{alternate goal objects}. When \textit{spamming}, participants would answer every question with the same answer; e.g., always answering \textit{no}. \textit{Absolute contradiction} was when participants determined the correct answer to the question-player's query and then provided the negation of this. Finally, \textit{alternate goal objects} describes the strategy of selecting an incorrect object in the image and providing answers as if this object was the correct goal. Of these, \textit{spamming} is fairly easy to automatically detect; i.e., by searching for games where all answers are identical. We find 19\% of the collected non-cooperative dialogues contain entirely \textit{spam} answers. This, of course, does not account for mixed strategies within a game, but it does indicate the dataset is not dominated by the least complex strategy. Lastly, we remind the reader, some non-cooperative strategies directly describe violations of Gricean maxims. In particular, \textit{absolute contradiction} and \textit{alternate goal objects} violate the maxim of quality, while \textit{spamming} violates the maxim of relevance. Due to the answer-player's simple vocabulary and the greater control given to the question-player (i.e., in directing conversation topic and length), the maxims of manner and quantity are difficult for the answer-player to violate. So, it is expected observed strategies do not violate these maxims.
\paragraph{Modeling Human Non-Cooperation.}
We further studied strategies in the autonomous non-cooperative answer-players. Notice, besides \textit{spamming}, the human strategies may require knowledge of the full dialogue history as well as other objects in the image. We tested whether the autonomous answer-player utilized this information by training multiple answer-players with different information access: the first produced answers conditioned only on the goal-object and the most immediate question (1), the next two were also conditioned on the full dialogue-history (2) \textit{or} the full image (3), and the last was conditioned on all of these features (4). We paired these non-cooperative answer-players with a question-player whose communication strategy focused on the object-identification task; i.e., using Eqn.~\eqref{eqn:rew_ours}. Answer-players~2,~3,~and~4 induced an object-identification error outside a 95\% confidence interval\footnote{An upper bound on true error induced by the 1st answer-player is 0.749 with confidence 95\% (Hoeffding Bound $\approx$ 10K samples). The sample error of the 2nd, 3rd, and 4th answer-player are, respectively, 0.756, 0.757, and 0.752.} of answer-player~1. In contrast, \citet{strub2017end} found that cooperative answer-players only needed access to the goal-object and the most immediate question to perform well. This result indicates the complexities inherent to deception and suggests that distinct strategies were learned when non-cooperative answer-players had access to more information. In the remainder, we focus on non-cooperative answer-player 2 with access to the full dialogue history. Our interpretation for answer-players 1, 3, and 4 is largely similar.
\paragraph{Empirical Validity of Def.~\ref{defn:only_def}.}
Our next observation concerns the formal definition of \textit{effective} given in Section~\ref{sec:theory} Def.~\ref{defn:only_def}. While the limiting property required by the definition is not easy to measure empirically, we observe in Figure~\ref{fig:main-result} that the object-identification error on \textit{non-cooperative} examples is relatively stable across question-player communication strategies. This fact -- that the non-cooperative answer-player exhibits behavior consistent with an \textit{effective} answer-player -- points to the validity of our theory. Recall, an effective answer-player is assumed in Prop.~\ref{prop:method_bound}. 
\paragraph{Empirical Validity of Proposition~\ref{prop:method_bound}.}
Finally, the primary conclusion of our theoretical analysis was that communication strategies which focus \textit{only} on the object-identification task should be effective for \textit{both} object-identification and cooperation-identification. Figure~\ref{fig:main-result} confirms this. Selecting a communication strategy based on improving object-identification improves object-identification as expected. Further, on the potentially opposing objective of identifying non-cooperation, this strategy is also effective. It far improves over a random baseline and also improves over the baseline which uses no RL-based strategy. On the other hand, the communication strategy which focuses only on the identification of non-cooperation fails at the opposing task of object-identification. This strategy performs almost as badly as a random baseline when the percent of non-cooperative examples is large and is also consistently worse than the baseline which uses no RL. The mixture of both strategies seems to achieve good middle ground. Recall, while this strategy may be heuristically intuited, our theoretical results formally justified this strategy as well. 
\section{Conclusion}
\label{sec:discussion}
Combining tools from learning theory, reinforcement learning, and supervised learning, we model partially non-cooperative communicative strategies in dialogue. Understanding such strategies is essential when building robust agents capable of conversing with parties of varying intent. Our theoretical and empirical findings suggest non-cooperative agents may sufficiently reveal themselves through their non-cooperative communicative behavior.

Although the dialogue game studied is simple, the results have ramifications for more complex dialogue systems. Our theoretical results, in particular, are not limited in this sense and may apply to designing communication strategies in distinct contexts. As noted in Section~\ref{sec:setup_applicab}, the limited assumptions we make facilitate this. For example, classifying intents and asking the right clarification questions is crucial to decision making in dialogue \citep{purver2003means, devault2007managing, khalid2020combining}. Our theory is directly applicable to this setting and could be applied to inform learning objectives for any dialogue agent that asks clarification questions to make a classification.
A real-world example of this is the online-banking setting studied by \citet{dhole2020resolving}, in which the dialogue agent asks clarification questions to decide the type of account a user would like to open. If we suppose some users may be non-cooperative in this context, our theoretical setup is satisfied: there is some feature space (the dialogues), the label space of user-intents is finite, users are labeled with a binary indicator of cooperation, and the dialogue agent can control the distribution over which it learns by asking clarification questions.
Our theoretical results should apply to many similar dialogue systems that can ask clarification questions or other types of questions. The only stipulations are that the theoretical setup is satisfied (e.g., in the manner just shown) and that our proposed assumptions on the nature of non-cooperative dialogue still hold (i.e., see Section~\ref{sec:comm_analysis}, Def.~\ref{defn:only_def}). 

To promote continued research, the collected corpus as well as our code are \href{https://github.com/anthonysicilia/modeling-non-cooperation-tacl2022}{publicly available}.\footnote{\url{https://github.com/anthonysicilia/modeling-non-cooperation-tacl2022}}

\section{Ethical Considerations}
We have described a research prototype. The proposed dataset does not include sensitive or personal data. Our human subject board approved our protocol. Human subjects participated voluntarily and were compensated fairly for their time. The publicly available dataset is fully anonymized. 

 The proposed architecture relies on pretrained models such as word or image embeddings so any harm or bias associated with these models may be present in our model. We believe general methods that propose to mitigate harms can resolve these issues.
 
 \section*{Acknowledgement}
 We would like to thank Matthew Stone, Raquel Fernandez, Katherine Atwell, and the anonymous reviewers for their helpful comments and suggestions. We also thank the action editors: Cindy Robinson, Ani Nenkova, and Claire Gardent.

\typeout{}
\bibliography{tacl2018}
\bibliographystyle{acl_natbib}

\end{document}